\documentclass{article}
\usepackage[utf8]{inputenc}
\usepackage[T1]{fontenc}
\usepackage{amsmath, amssymb, amsthm}
\usepackage{booktabs}
\usepackage{graphicx}
\usepackage{tikz}
\usetikzlibrary{shapes.geometric, arrows, positioning}
\usepackage{hyperref}
\hypersetup{
  colorlinks=true,
  linkcolor=blue,
  citecolor=blue,
  urlcolor=blue,
  pdfstartview=FitH,
  pdfpagemode=UseOutlines
}
\usepackage{bookmark}
\usepackage[preprint]{neurips_2025}
\setcitestyle{numbers,square}
\usepackage{geometry}
\usepackage{listings}
\usepackage{enumitem}
\usepackage{adjustbox}

\theoremstyle{plain}
\newtheorem{theorem}{Theorem}
\newtheorem{lemma}{Lemma}
\newtheorem{corollary}{Corollary}
\newtheorem{proposition}{Proposition}
\theoremstyle{definition}
\newtheorem{assumption}{Assumption}
\newtheorem{definition}{Definition}

\title{Eliciting Fine-Tuned Transformer Capabilities via Inference-Time Techniques}
\author{
  Asankhaya Sharma \\
  Patched Codes, Inc. \\
  \texttt{asankhaya@patchedcodes.com}
}

\begin{document}

\maketitle

\begin{abstract}
Large language models have transformed natural language processing, yet supervised fine-tuning (SFT) remains computationally intensive. This paper formally proves that capabilities acquired through SFT can be approximated by a base transformer model using inference-time techniques, specifically in-context learning (ICL), without altering model parameters, under idealized assumptions including unbounded computational resources and access to the fine-tuning dataset. We extend these results to practical scenarios with finite context lengths and partial dataset access. For text generation tasks with fixed output length \( l \), datasets of size \(\mathcal{O}\left( \frac{m V}{\varepsilon^2} \log \frac{m}{\delta} \right)\) or, with bounded context, \(\mathcal{O}\left( \frac{l \log V}{\varepsilon^2} \log \frac{1}{\delta} \right)\) suffice to approximate fine-tuned behavior across \( m \) contexts within error \( \varepsilon \), where \( V \) is the vocabulary size and \( \delta \) is the failure probability. For linear classification, datasets of size \(\mathcal{O}\left( \frac{d}{\varepsilon} \right)\) or, with fixed context, \(\mathcal{O}\left( \frac{1}{\varepsilon^2} \log \frac{1}{\delta} \right)\) are sufficient, where \( d \) is the input dimension. Grounded in the Turing completeness of transformers, these results provide a theoretical foundation for resource-efficient deployment of large language models, with practical techniques like retrieval-augmented generation bridging theory to real-world applications.
\end{abstract}

\section{Introduction}
\label{sec:intro}

Transformer models, introduced by Vaswani et al. \cite{vaswani2017attention}, are the backbone of natural language processing (NLP), powering large language models (LLMs) like DeepSeek-R1 \cite{deepseek2025r1} and Claude 4 \cite{anthropic2025claude}. These models use self-attention to capture long-range dependencies, achieving breakthroughs in tasks such as language modeling, translation, and text generation. However, supervised fine-tuning (SFT) to adapt pre-trained transformers to specific tasks is computationally expensive, often requiring thousands of GPU hours \cite{deepseek2025r1}. This prompts a key question: Can the capabilities gained through SFT be elicited from a base transformer using inference-time techniques, such as in-context learning (ICL), without parameter updates?

ICL and prompting allow models to adapt to tasks by conditioning on input-output examples \cite{brown2020language, xie2021explanation}. If fine-tuned capabilities are latent in the base model, SFT may primarily refine access to pre-existing knowledge \cite{zhang2025sft, elhage2023eliciting}. This paper provides a formal proof that, under idealized conditions (unbounded computational resources and access to the fine-tuning dataset), a base transformer can approximate SFT capabilities via ICL within a quantifiable error margin. We extend these results to practical settings with finite context lengths and partial dataset access, demonstrating that minimal datasets can approximate fine-tuned behavior. Specifically, for text generation, a dataset of size \(\mathcal{O}\left( \frac{m V}{\varepsilon^2} \log \frac{m}{\delta} \right)\) or, with fixed context and output length \( l \), \(\mathcal{O}\left( \frac{l \log V}{\varepsilon^2} \log \frac{1}{\delta} \right)\) approximates fine-tuned distributions across \( m \) contexts, where \( V \) is the vocabulary size, \( \varepsilon \) is the error tolerance, and \( \delta \) is the failure probability. For linear classification, datasets of size \(\mathcal{O}\left( \frac{d}{\varepsilon} \right)\) or, with fixed context, \(\mathcal{O}\left( \frac{1}{\varepsilon^2} \log \frac{1}{\delta} \right)\) suffice, where \( d \) is the input dimension. Rooted in the Turing completeness of transformers \cite{bhattamishra2020computational}, these results establish a framework for resource-efficient LLM deployment, with practical approximations like retrieval-augmented generation (RAG) enhancing real-world applicability.

Figure \ref{fig:overview} illustrates our approach, showing how a base model, prompted with a dataset \( D \), approximates the fine-tuned model’s output distribution.

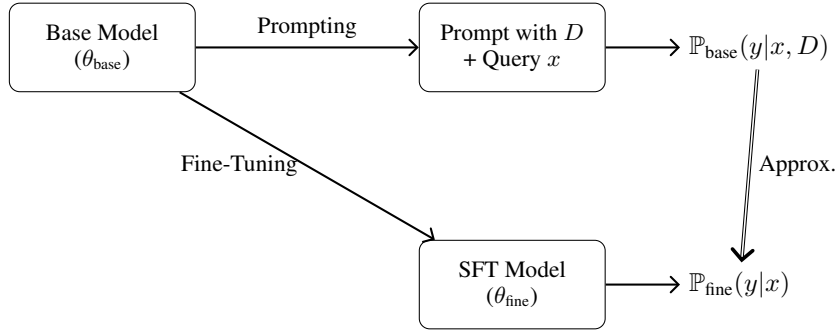
\begin{figure}[ht]
    \centering
    \resizebox{0.8\columnwidth}{!}{
    \begin{tikzpicture}[
        node distance=2cm and 1cm,
        box/.style={rectangle, draw, rounded corners, minimum height=1.2cm, minimum width=2.5cm, align=center, font=\small},
        arrow/.style={-angle 90, thick}
    ]
    \node[box] (base) {Base Model \\ (\( \theta_{\text{base}} \))};
    \node[box, right=3cm of base] (prompt) {Prompt with \( D \) \\ + Query \( x \)};
    \node[box, below=2cm of prompt] (fine) {SFT Model \\ (\( \theta_{\text{fine}} \))};
    \node[right=1cm of prompt] (pbase) {\( \mathbb{P}_{\text{base}}(y|x, D) \)};
    \node[right=1cm of fine] (pfine) {\( \mathbb{P}_{\text{fine}}(y|x) \)};
    \draw[arrow] (base) -- node[above, font=\small] {Prompting} (prompt);
    \draw[arrow] (prompt) -- (pbase);
    \draw[arrow] (base) -- node[left, font=\small] {Fine-Tuning} (fine);
    \draw[arrow] (fine) -- (pfine);
    \draw[double, -angle 90] (pbase) -- node[right, font=\small] {Approx.} (pfine);
    \end{tikzpicture}
    }
    \caption{Overview of eliciting fine-tuned capabilities. A base model, prompted with dataset \( D \) and query \( x \), produces an output distribution \( \mathbb{P}_{\text{base}}(y|x, D) \) approximating the fine-tuned model’s distribution \( \mathbb{P}_{\text{fine}}(y|x) \).}
    \label{fig:overview}
\end{figure}

Our contributions are:
\begin{enumerate}
    \item A proof that base transformers can approximate SFT capabilities via ICL under unbounded resources, within error \( \varepsilon \) (Section \ref{sec:proof}).
    \item Demonstrations that text generation can be approximated with datasets of size \(\mathcal{O}\left( \frac{m V}{\varepsilon^2} \log \frac{m}{\delta} \right)\) or, with fixed context, \(\mathcal{O}\left( \frac{l \log V}{\varepsilon^2} \log \frac{1}{\delta} \right)\) (Sections \ref{subsec:textgen}, \ref{subsec:bounded-textgen}).
    \item Proofs that linear classifiers can be approximated with datasets of size \(\mathcal{O}\left( \frac{d}{\varepsilon} \right)\) or, with fixed context, \(\mathcal{O}\left( \frac{1}{\varepsilon^2} \log \frac{1}{\delta} \right)\) (Sections \ref{subsec:minimal}, \ref{subsec:bounded-linear}).
    \item Practical techniques, such as RAG and few-shot prompting, to bridge theoretical results to applications (Section \ref{sec:discussion}).
\end{enumerate}

\section{Related Work}
\label{sec:related}

\subsection{Computational Power of Transformers}
Transformers with unbounded resources are Turing-complete \cite{bhattamishra2020computational}, with self-attention simulating Turing machine tapes and feed-forward layers encoding transition rules \cite{perez2021attention}. While fixed-depth transformers may not be Turing-complete \cite{lifeiscomputation2024}, modified architectures achieve this \cite{upadhyay2024turing}. As universal approximators \cite{yun2020are}, transformers support our claim that fine-tuned behaviors can be approximated without parameter updates \cite{vapnik1998}.

\subsection{In-Context Learning}
ICL enables task adaptation by conditioning on input-output examples \cite{brown2020language}. For example, sentiment-labeled reviews allow label prediction for new inputs. DeepSeek-R1 shows robust zero-shot ICL \cite{deepseek2025r1}, and structured prompts enhance reasoning \cite{kojima2022large}. ICL is modeled as Bayesian inference \cite{xie2021explanation}, with pretrained transformers generalizing well \cite{han2021pretrained}. Instruction tuning boosts zero-shot performance \cite{wei2021finetuned}, forming the basis for our framework.

\subsection{Fine-Tuning and Latent Knowledge}
SFT refines latent capabilities in LLMs. Early layers retain general knowledge \cite{phang2021fine}, while the tuned lens reveals latent predictions \cite{elhage2023eliciting}. Low-rank adaptation (LoRA) shows minimal updates suffice \cite{hu2021lora}. Recent work suggests SFT reformats outputs for task-specific styles \cite{zhang2025sft}, supporting our hypothesis that base models can approximate fine-tuned behaviors via inference.

\subsection{Inference-Time Alternatives to Fine-Tuning}
Inference-time techniques offer alternatives to SFT. Chain-of-thought prompting elicits reasoning \cite{wei2022chain}, and test-time compute scaling improves outcomes \cite{snell2024scaling}. ICL-as-programming treats prompts as programs \cite{garg2022can}. Optimized inference for software development \cite{sharma2024moa, sharma2024rtc} and pivotal token search \cite{pts} show practical advancements. Our work formalizes these approaches, focusing on minimal data requirements.

\section{Preliminaries}
\label{sec:prelim}

\begin{definition}[Transformer Model]
\label{def:transformer}
A transformer model \( M \) with parameters \( \theta \) maps an input sequence \( x \in \mathcal{X} \) (token sequences) to an output distribution \( \mathbb{P}(y|x; \theta) \), where \( y \in \mathcal{Y} \). Transformers use stacked layers with self-attention and feed-forward components \cite{vaswani2017attention}. The base model is \( M_{\text{base}} \) with parameters \( \theta_{\text{base}} \), and the fine-tuned model is \( M_{\text{fine}} \) with parameters \( \theta_{\text{fine}} \).
\end{definition}

\begin{definition}[Supervised Fine-Tuning]
\label{def:fine-tuning}
SFT updates \( \theta_{\text{base}} \to \theta_{\text{fine}} = \theta_{\text{base}} + \Delta \theta \) by minimizing cross-entropy loss on a dataset \( D = \{(x_i, y_i)\}_{i=1}^N \subset \mathcal{X} \times \mathcal{Y} \), yielding \( \mathbb{P}_{\text{fine}}(y|x) \).
\end{definition}

\begin{definition}[Inference Technique]
\label{def:inference}
An inference technique \( T \) applied to \( M_{\text{base}} \) at test time produces \( \mathbb{P}_{\text{base}}(y|x, T) \). Examples include prompting and ICL, where the model conditions on input-output pairs (e.g., predicting `positive' for a review given labeled examples).
\end{definition}

\begin{assumption}[Unbounded Computational Resources]
\label{assump:unbounded}
\( M_{\text{base}} \) has infinite context length and computational resources, approximated by large context windows (e.g., 1M tokens in GPT-4.1 \cite{openai2025gpt41}) and scalable compute \cite{snell2024scaling}.
\end{assumption}

\begin{assumption}[Turing Completeness]
\label{assump:turing}
Transformers with unbounded resources are Turing-complete \cite{bhattamishra2020computational}, capable of simulating any computable function via self-attention \cite{perez2021attention}.
\end{assumption}

\begin{assumption}[Access to Fine-Tuning Dataset]
\label{assump:access}
The fine-tuning dataset \( D = \{(x_i, y_i)\}_{i=1}^N \) is accessible, reflecting scenarios where fine-tuning data is available for inference-time use.
\end{assumption}

\section{Main Theorem}
\label{sec:theorem}

\begin{theorem}
\label{thm:main}
Under Assumptions \ref{assump:unbounded}, \ref{assump:turing}, and \ref{assump:access}, for any fine-tuned model \( M_{\text{fine}} \) derived from \( M_{\text{base}} \) via SFT, and for any \(\varepsilon > 0\), there exists an inference technique \( T \) and dataset size \( N \) such that, for all \( x \in \mathcal{X}, y \in \mathcal{Y} \), the total variation distance between the base and fine-tuned output distributions satisfies:
\begin{equation}
\text{TV}(\mathbb{P}_{\text{base}}(y|x, T), \mathbb{P}_{\text{fine}}(y|x)) \leq \varepsilon,
\label{eq:sft}
\end{equation}
with \(\varepsilon = \mathcal{O}(1 / \sqrt{N})\) for typical tasks, decreasing as computational resources (e.g., context length) and dataset size increase.
\end{theorem}

\section{Proof of Theorem \ref{thm:main}}
\label{sec:proof}

We prove Theorem \ref{thm:main} by constructing an inference technique \( T \) that enables \( M_{\text{base}} \) to approximate the output distribution of \( M_{\text{fine}} \) within error \( \varepsilon \). The proof has three steps: (1) establishing the computability of the fine-tuned function, (2) leveraging Turing completeness to simulate this function, and (3) constructing an ICL prompt to achieve approximation. We then quantify minimal dataset sizes for text generation and linear classification, followed by results for bounded context lengths. Figure \ref{fig:turing} illustrates the Turing machine simulation, Table \ref{tab:techniques} summarizes the inference technique, and Figure \ref{fig:prompt} depicts the prompt structure.

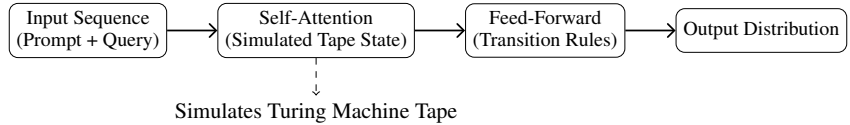
\begin{figure}[ht]
    \centering
    \resizebox{0.8\columnwidth}{!}{
    \begin{tikzpicture}[
        node distance=1.2cm and 0.8cm,
        box/.style={rectangle, draw, rounded corners, minimum height=0.7cm, minimum width=1.3cm, align=center, font=\small},
        arrow/.style={-angle 90, thick}
    ]
    \node[box] (input) {Input Sequence \\ (Prompt + Query)};
    \node[box, right=of input] (attention) {Self-Attention \\ (Simulated Tape State)};
    \node[box, right=of attention] (ffn) {Feed-Forward \\ (Transition Rules)};
    \node[box, right=of ffn] (output) {Output Distribution};
    \draw[arrow] (input) -- (attention);
    \draw[arrow] (attention) -- (ffn);
    \draw[arrow] (ffn) -- (output);
    \node[below=0.6cm of attention] (tape) {Simulates Turing Machine Tape};
    \draw[dashed, ->] (attention) -- (tape);
    \end{tikzpicture}
    }
    \caption{Self-attention manages the simulated tape state of a Turing machine, with feed-forward layers encoding transition rules (Lemma \ref{lemma:turing}).}
    \label{fig:turing}
\end{figure}

\begin{table}[ht]
    \centering
    \caption{Inference technique for approximating SFT capabilities.}
    \label{tab:techniques}
    \resizebox{0.8\columnwidth}{!}{
    \begin{tabular}{lcc}
        \toprule
        \textbf{Fine-Tuning Type} & \textbf{Inference Technique \( T \)} & \textbf{Key Mechanism} \\
        \midrule
        Supervised Fine-Tuning & Prompting with dataset \( D \) & In-context learning \\
        \bottomrule
    \end{tabular}
    }
\end{table}

\subsection{Computability of Fine-Tuned Functions}

\begin{lemma}
\label{lemma:computable}
The fine-tuned function \( f_{\text{fine}}: \mathcal{X} \to \mathcal{Y} \), defined as \( f_{\text{fine}}(x) = \arg\max_{y \in \mathcal{Y}} \mathbb{P}_{\text{fine}}(y|x) \), is computable.
\end{lemma}

\begin{proof}
For SFT, \( M_{\text{fine}} \) is trained on a finite dataset \( D = \{(x_i, y_i)\}_{i=1}^N \), minimizing:
\begin{equation}
\mathcal{L} = -\frac{1}{N} \sum_{i=1}^N \log \mathbb{P}_{\text{fine}}(y_i | x_i).
\label{eq:sft-loss}
\end{equation}
Gradient descent, a computable algorithm, updates \( \theta_{\text{base}} \to \theta_{\text{fine}} \). Since \( D \) is finite and the transformer has fixed parameters, \( f_{\text{fine}} \) is computable via finite operations \cite{vapnik1998}.
\end{proof}

\subsection{Turing Completeness of the Base Model}

\begin{lemma}
\label{lemma:turing}
Under Assumption \ref{assump:turing}, \( M_{\text{base}} \) with parameters \( \theta_{\text{base}} \) can simulate any Turing machine \( TM_f \) computing a function \( f \).
\end{lemma}

\begin{proof}
Transformers with unbounded resources are Turing-complete \cite{bhattamishra2020computational}. Self-attention:
\begin{equation}
\text{Attention}(Q, K, V) = \text{softmax}\left( \frac{Q K^T}{\sqrt{d_k}} \right) V,
\label{eq:attention}
\end{equation}
manages a tape by weighting tokens, simulating state transitions. Feed-forward layers:
\begin{equation}
\text{FFN}(x) = \max(0, x W_1 + b_1) W_2 + b_2,
\label{eq:ffn}
\end{equation}
encode transition rules, with residual connections propagating information. Given the transition table of \( TM_f \) and the input sequence, \( M_{\text{base}} \) computes \( f \) \cite{perez2021attention}.
\end{proof}

\subsection{Construction of Inference Technique}

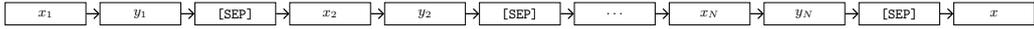
\begin{figure}[ht]
    \centering
    \resizebox{\columnwidth}{!}{
    \begin{tikzpicture}[
        node distance=0.3cm and 0.3cm,
        box/.style={rectangle, draw, minimum height=0.5cm, minimum width=1.8cm, align=center, font=\footnotesize},
        arrow/.style={-angle 90, thick}
    ]
    \node[box] (x1) {\( x_1 \)};
    \node[box, right=of x1] (y1) {\( y_1 \)};
    \node[box, right=of y1] (sep1) {\texttt{[SEP]}};
    \node[box, right=of sep1] (x2) {\( x_2 \)};
    \node[box, right=of x2] (y2) {\( y_2 \)};
    \node[box, right=of y2] (sep2) {\texttt{[SEP]}};
    \node[box, right=of sep2] (dots) {\( \ldots \)};
    \node[box, right=of dots] (xn) {\( x_N \)};
    \node[box, right=of xn] (yn) {\( y_N \)};
    \node[box, right=of yn] (sepn) {\texttt{[SEP]}};
    \node[box, right=of sepn] (x) {\( x \)};
    \draw[arrow] (x1) -- (y1);
    \draw[arrow] (y1) -- (sep1);
    \draw[arrow] (sep1) -- (x2);
    \draw[arrow] (x2) -- (y2);
    \draw[arrow] (y2) -- (sep2);
    \draw[arrow] (sep2) -- (dots);
    \draw[arrow] (dots) -- (xn);
    \draw[arrow] (xn) -- (yn);
    \draw[arrow] (yn) -- (sepn);
    \draw[arrow] (sepn) -- (x);
    \node[above=0.3cm of x1] (prompt) {\textbf{Prompt Sequence}};
    \end{tikzpicture}
    }
    \caption{Prompt structure for \( T_{\text{SFT}} \), where input-output pairs from dataset \( D \) are concatenated with query \( x \) (Lemma \ref{lemma:sft}).}
    \label{fig:prompt}
\end{figure}

\begin{proposition}
\label{prop:in-context}
ICL enables \( M_{\text{base}} \) to model a task distribution \( \mathbb{P}(y|x, D) \) by processing a prompt with dataset \( D = \{(x_i, y_i)\}_{i=1}^N \), approximating \( \mathbb{P}_{\text{fine}}(y|x) \).
\end{proposition}

\begin{proof}
ICL acts as Bayesian inference over task distributions \cite{xie2021explanation}. For a prompt \( p = [x_1, y_1, \ldots, x_N, y_N, x] \), \( M_{\text{base}} \) computes:
\begin{equation}
\mathbb{P}_{\text{base}}(y | x, p) = \mathbb{P}_{\text{base}}(y | x_1, y_1, \ldots, x_N, y_N, x).
\label{eq:in-context}
\end{equation}
Since \( D \) defines the task, \( M_{\text{base}} \) infers the mapping implied by \( \mathbb{P}_{\text{fine}} \) by generalizing from examples \cite{han2021pretrained}.
\end{proof}

\begin{lemma}
\label{lemma:sft}
For SFT, there exists an inference technique \( T_{\text{SFT}} \) such that, for any \(\varepsilon > 0\), there exists a dataset size \( N \) satisfying:
\[
\text{TV}(\mathbb{P}_{\text{base}}(y|x, T_{\text{SFT}}), \mathbb{P}_{\text{fine}}(y|x)) \leq \varepsilon, \quad \forall x \in \mathcal{X}, y \in \mathcal{Y},
\]
with \(\varepsilon = \mathcal{O}(1 / \sqrt{N})\) for typical tasks, where \(\varepsilon\) decreases as computational resources (e.g., context length) increase.
\end{lemma}

\begin{proof}
Define \( T_{\text{SFT}} \) as the prompt:
\begin{equation}
p = [x_1, y_1, x_2, y_2, \ldots, x_N, y_N, x],
\label{eq:sft-prompt}
\end{equation}
where \( (x_i, y_i) \in D \). Under Assumption \ref{assump:unbounded}, \( M_{\text{base}} \) processes the entire prompt, requiring a context length proportional to \( N \). By Assumption \ref{assump:access}, \( D \) provides examples for the task, but since \( D \) is finite, \( \mathbb{P}_{\text{fine}}(y|x) \) generalizes beyond \( D \) via SFT optimization, while ICL infers patterns from examples. By Proposition \ref{prop:in-context}, ICL approximates:
\[
\text{TV}(\mathbb{P}_{\text{base}}(y|x, p), \mathbb{P}_{\text{fine}}(y|x)) \leq \varepsilon,
\]
where \(\varepsilon\) depends on the base model’s capacity, \( N \), and context length. For typical tasks, ICL’s sample complexity suggests \(\varepsilon = \mathcal{O}(1 / \sqrt{N})\) \cite{xie2021explanation}, with faster decay (e.g., \(\mathcal{O}(1 / N)\)) possible for simpler tasks under uniform data distributions. As \( N \to \infty \) and context length grows, ICL’s generalization improves, reducing \(\varepsilon \to 0\). Prompt structure (e.g., example order) may introduce variability, potentially increasing \(\varepsilon\), addressed in Section \ref{sec:discussion}.
\end{proof}

\subsection{Adapting to Practical Scenarios}
\label{subsec:practical}

Assumptions \ref{assump:unbounded} and \ref{assump:access} are idealized. In practice, context lengths are finite (e.g., 1M tokens \cite{openai2025gpt41}), and access to \( D \) may be partial. Using a subset \( D' \subset D \) of size \( |D'| = o(N) \), the approximation error increases by:
\[
\text{TV}(\mathbb{P}_{\text{base}}(y|x, D'), \mathbb{P}_{\text{fine}}(y|x)) \leq \varepsilon + \mathcal{O}(1 / \sqrt{|D'|}),
\]
derived as follows: Assuming i.i.d. samples in \( D \), the empirical distribution from \( D' \) deviates from the true distribution by \(\mathcal{O}(1 / \sqrt{|D'|})\) in total variation distance, by Hoeffding’s inequality \cite{feldman2020}. For non-i.i.d. data (e.g., text), this bound may weaken, requiring larger \( |D'| \). Techniques like RAG \cite{lewis2020rag} select representative examples using similarity metrics (e.g., cosine distance in embeddings), ensuring \( D' \) captures key task patterns, though exact error reduction depends on the task. Finite context limits the number of examples, addressed in Section \ref{subsec:bounded-context}. These relaxations enhance practical applicability, though optimal subset selection remains an open challenge.

\subsection{Minimal Datasets}
\label{subsec:minimal-datasets}

\subsubsection{Minimal Dataset for Text Generation}
\label{subsec:textgen}

\begin{theorem}
\label{thm:textgen}
Let \( \mathcal{C} = \{ c_1, \ldots, c_m \} \) be contexts, and let \( M_{\text{fine}} \) define next-token distributions \( p_{\text{fine}}(x_{t+1} | c_i) \) over vocabulary size \( V \). There exists a dataset \( D' \), with samples \( (c_i, x_{t+1}) \sim p_{\text{fine}}(\cdot | c_i) \), such that when \( M_{\text{base}} \) is prompted with \( D' \), it satisfies:
\[
\sup_{c \in \mathcal{C}} \| p_{\text{base}}(x_{t+1} | c, D') - p_{\text{fine}}(x_{t+1} | c) \|_1 \leq \varepsilon + \eta,
\]
with probability at least \( 1 - \delta \), where \(\eta\) is the ICL approximation error, and:
\[
|D'| = \mathcal{O}\left( \frac{m V}{\varepsilon^2} \log \frac{m}{\delta} \right).
\]
\end{theorem}

\begin{proof}
For each \( c_i \in \mathcal{C} \), approximate \( p_{\text{fine}}(x_{t+1} | c_i) \) with \( p_{\text{base}}(x_{t+1} | c_i, D') \) within total variation distance:
\[
\| p_{\text{base}}(\cdot | c_i, D') - p_{\text{fine}}(\cdot | c_i) \|_1 = \sum_{v \in V} | p_{\text{base}}(v | c_i, D') - p_{\text{fine}}(v | c_i) | \leq \varepsilon + \eta.
\]
With \( n_i \) samples \( \{ x_{t+1}^{(j)} \}_{j=1}^{n_i} \sim p_{\text{fine}}(\cdot | c_i) \), the empirical distribution is:
\[
\hat{p}(v | c_i) = \frac{1}{n_i} \sum_{j=1}^{n_i} \mathbb{I}\{ x_{t+1}^{(j)} = v \}.
\]
By Hoeffding’s inequality \cite{vapnik1998}, \( \| \hat{p}(\cdot | c_i) - p_{\text{fine}}(\cdot | c_i) \|_1 \leq \varepsilon \) with probability \( 1 - \delta_i \) requires:
\[
n_i = \mathcal{O}\left( \frac{V}{\varepsilon^2} \log \frac{1}{\delta_i} \right).
\]
For \( m \) contexts, set \( \delta_i = \frac{\delta}{m} \). The union bound ensures:
\[
\mathbb{P}\left( \bigcup_{i=1}^m \{ \| \hat{p}(\cdot | c_i) - p_{\text{fine}}(\cdot | c_i) \|_1 > \varepsilon \} \right) \leq \sum_{i=1}^m \delta_i = \delta.
\]
Thus:
\[
n_i = \mathcal{O}\left( \frac{V}{\varepsilon^2} \log \frac{m}{\delta} \right).
\]
The total dataset size is:
\[
|D'| = m \cdot n_i = \mathcal{O}\left( \frac{m V}{\varepsilon^2} \log \frac{m}{\delta} \right).
\]
The base model’s ICL introduces an error \(\eta\), assumed small but dependent on model capacity and prompt design (e.g., suboptimal ordering increases \(\eta\)) \cite{xie2021explanation}. Assuming \( M_{\text{base}} \) approximates \( \hat{p}(\cdot | c_i) \) within \(\eta\), the total error is \(\varepsilon + \eta\). ICL imperfections may increase \( n_i \) by a constant factor, discussed in Section \ref{sec:discussion}.
\end{proof}

\subsubsection{Minimal Dataset for Linear Classification}
\label{subsec:minimal}

\begin{theorem}
\label{thm:minimal}
For a binary classification task where \( M_{\text{fine}} \) is a linear classifier trained on \( D \subseteq \mathbb{R}^d \times \{0, 1\} \), there exists a subset \( D' \subseteq D \) with size \( |D'| = \mathcal{O}\left( \frac{d}{\varepsilon} \right) \), such that when \( M_{\text{base}} \) is prompted with \( D' \), the output distribution satisfies:
\[
\sup_{x \in \mathbb{R}^d} \left| \mathbb{P}_{\text{base}}(y | x, D') - \mathbb{P}_{\text{fine}}(y | x) \right| \leq \varepsilon + \eta,
\]
where \(\eta\) is the ICL approximation error, assuming \( M_{\text{base}} \) approximates linear classifiers via ICL within error \( \varepsilon/2 + \eta \).
\end{theorem}

\begin{proof}
\begin{sloppypar}
Let \( M_{\text{fine}} \) minimize the logistic loss on \( D = \{(x_i, y_i)\}_{i=1}^N \):
\begin{align}
\mathcal{L}(w, b) &= \frac{1}{N} \sum_{i=1}^N \log(1 + \exp(-y_i (w^T x_i + b))),
\label{eq:logistic-loss}
\end{align}
yielding \( \mathbb{P}_{\text{fine}}(y=1 | x) = \sigma(w^T x + b) \), where \( \sigma(z) = (1 + e^{-z})^{-1} \). By coreset theory \cite{coreset2018}, there exists \( D' \subseteq D \) with:
\[
|D'| = \mathcal{O}\left( \frac{d}{\varepsilon} \right),
\]
such that a classifier trained on \( D' \), with parameters \( (w', b') \), satisfies:
\[
\sup_{x \in \mathbb{R}^d} \left| \sigma(w'^T x + b') - \sigma(w^T x + b) \right| \leq \varepsilon/2.
\]
Assume \( M_{\text{base}} \), prompted with \( D' \), approximates the classifier with parameters \( (w', b') \) via ICL \cite{xie2021explanation}, such that:
\[
\left| \mathbb{P}_{\text{base}}(y=1 | x, D') - \sigma(w'^T x + b') \right| \leq \varepsilon/2 + \eta,
\]
where \(\eta\) accounts for ICL errors (e.g., from model capacity or prompt design). The total error is:
\begin{align}
&\left| \mathbb{P}_{\text{base}}(y=1 | x, D') - \mathbb{P}_{\text{fine}}(y=1 | x) \right| \notag \\
&\leq \left| \mathbb{P}_{\text{base}}(y=1 | x, D') - \sigma(w'^T x + b') \right| + \left| \sigma(w'^T x + b') - \sigma(w^T x + b) \right| \notag \\
&\leq (\varepsilon/2 + \eta) + \varepsilon/2 = \varepsilon + \eta.
\label{eq:error-bound}
\end{align}
Since \( \mathbb{P}_{\text{base}}(y=0 | x, D') = 1 - \mathbb{P}_{\text{base}}(y=1 | x, D') \), the total variation distance is at most \(\varepsilon + \eta\). ICL imperfections may increase \( |D'| \) slightly, discussed in Section \ref{sec:discussion}.
\end{sloppypar}
\end{proof}

\subsection{Bounded Context Length}
\label{subsec:bounded-context}

We address settings where context length limits the number of prompt examples. Assumptions include sufficient transformer capacity and task simplicity, as complex tasks or weaker models may widen error bounds.

\subsubsection{Linear Classification with Fixed Context Length}
\label{subsec:bounded-linear}

\begin{theorem}
\label{thm:bounded-linear}
For a binary classification task where \( M_{\text{fine}} \) is a linear classifier trained on \( D \subseteq \mathbb{R}^d \times \{0, 1\} \), for each input \( x \), select a subset \( S_x \subseteq D \) of size \( |S_x| = \mathcal{O}\left( \frac{1}{\varepsilon^2} \log \frac{1}{\delta} \right) \), e.g., \( k \)-nearest neighbors to \( x \). When \( M_{\text{base}} \) is prompted with \( S_x \), the output distribution satisfies:
\[
\sup_{x \in \mathbb{R}^d} \left| \mathbb{P}_{\text{base}}(y | x, S_x) - \mathbb{P}_{\text{fine}}(y | x) \right| \leq \varepsilon + \eta,
\]
with probability at least \( 1 - \delta \), where \(\eta\) is the ICL approximation error, assuming \( M_{\text{base}} \) approximates the classifier trained on \( S_x \) within error \( \mathcal{O}(1 / \sqrt{|S_x|}) + \eta \).
\end{theorem}

\begin{proof}
Let \( M_{\text{fine}} \) have parameters \( (w, b) \), producing \( \mathbb{P}_{\text{fine}}(y=1 | x) = \sigma(w^T x + b) \), where \( \sigma(z) = (1 + e^{-z})^{-1} \). For a query \( x \), select a subset \( S_x \subseteq D \) of size:
\[
k = \mathcal{O}\left( \frac{1}{\varepsilon^2} \log \frac{1}{\delta} \right),
\]
using, e.g., \( k \)-nearest neighbors in Euclidean distance. A classifier trained on \( S_x \), with parameters \( (w', b') \), approximates the decision boundary locally. The total error is bounded as:
\begin{align}
&\left| \mathbb{P}_{\text{base}}(y=1 | x, S_x) - \mathbb{P}_{\text{fine}}(y=1 | x) \right| \notag \\
&\quad \leq \left| \mathbb{P}_{\text{base}}(y=1 | x, S_x) - \sigma(w'^T x + b') \right| \notag \\
&\quad \quad + \left| \sigma(w'^T x + b') - \sigma(w^T x + b) \right|.
\end{align}
Assume \( M_{\text{base}} \) approximates the classifier on \( S_x \) via ICL within:
\[
\left| \mathbb{P}_{\text{base}}(y=1 | x, S_x) - \sigma(w'^T x + b') \right| \leq \mathcal{O}(1 / \sqrt{k}) + \eta,
\]
where \(\eta\) accounts for ICL errors (e.g., due to model capacity or prompt design) \cite{xie2021explanation}, and \(\mathcal{O}(1 / \sqrt{k})\) reflects ICL sample complexity for linear functions \cite{garg2022can}. The local classifier error is:
\[
\left| \sigma(w'^T x + b') - \sigma(w^T x + b) \right| = \mathcal{O}(1 / \sqrt{k}),
\]
assuming \( S_x \) captures the decision boundary \cite{vapnik1998}. Since \( k = \mathcal{O}\left( \frac{1}{\varepsilon^2} \log \frac{1}{\delta} \right) \), we have \( \mathcal{O}(1 / \sqrt{k}) \leq \varepsilon/2 \). Thus, with probability at least \( 1 - \delta \), the total error is:
\[
(\varepsilon/2 + \eta) + \varepsilon/2 = \varepsilon + \eta.
\]
This bound holds uniformly for all \( x \in \mathbb{R}^d \).
\end{proof}

\subsubsection{Text Generation with Fixed Context Length}
\label{subsec:bounded-textgen}

\begin{theorem}
\label{thm:bounded-textgen}
For a text generation task with fixed output length \( l \), let \( \mathcal{C} = \{ c_1, \ldots, c_m \} \) be contexts, and let \( M_{\text{fine}} \) define sequence distributions \( p_{\text{fine}}(x_1, \ldots, x_l | c_i) \) over vocabulary \( V \). For each context \( c \in \mathcal{C} \), select a subset \( S_c \) of size \( |S_c| = \mathcal{O}\left( \frac{l \log |V|}{\varepsilon^2} \log \frac{1}{\delta} \right) \), with samples \( (c_i, x_1^{(i)}, \ldots, x_l^{(i)}) \) where \( c_i \approx c \) and \( (x_1^{(i)}, \ldots, x_l^{(i)}) \sim p_{\text{fine}}(\cdot | c_i) \). When \( M_{\text{base}} \) is prompted with \( S_c \), it satisfies:
\[
\sup_{c \in \mathcal{C}} \| p_{\text{base}}(x_1, \ldots, x_l | c, S_c) - p_{\text{fine}}(x_1, \ldots, x_l | c) \|_1 \leq \varepsilon + \eta,
\]
with probability at least \( 1 - \delta \), where \(\eta\) is the ICL approximation error.
\end{theorem}

\begin{proof}
Treat text generation of length \( l \) as multi-label classification over \( |V|^l \) sequences. For each \( c \in \mathcal{C} \), select \( S_c \) of size \( k = \mathcal{O}\left( \frac{l \log |V|}{\varepsilon^2} \log \frac{1}{\delta} \right) \), e.g., similar contexts \( c_i \). The error is:
\[
\| p_{\text{base}}(x_1, \ldots, x_l | c, S_c) - p_{\text{fine}}(x_1, \ldots, x_l | c) \|_1.
\]
This decomposes into:
\begin{itemize}
    \item \textbf{ICL approximation error:} Bounded by $\mathcal{O}(1 / \sqrt{k}) + \eta$, where $\eta$ accounts for model capacity and prompt design \cite{xie2021explanation}.
    \item \textbf{Similarity error:} Bounded by $\mathcal{O}(1 / \sqrt{k})$ with similarity-based selection \cite{lewis2020rag}.
\end{itemize}

The multi-label problem reduces to $|V|^l$ binary decisions, with a union bound requiring:
\[
k = \mathcal{O}\left( \frac{\log |V|^l}{\varepsilon^2} \log \frac{|V|^l}{\delta} \right) = \mathcal{O}\left( \frac{l \log |V|}{\varepsilon^2} \log \frac{1}{\delta} \right).
\]
With \( k \) as above, the total error is \(\varepsilon + \eta\), with probability \( 1 - \delta \).
\end{proof}

\subsection{Generalization and Conclusion}

\begin{corollary}
\label{cor:general}
For any SFT capability of \( M_{\text{fine}} \), there exists an inference technique \( T \) such that \( M_{\text{base}} \) approximates that capability within error \( \varepsilon + \eta \).
\end{corollary}

\begin{proof}
By Lemma \ref{lemma:computable}, \( f_{\text{fine}} \) is computable. By Lemma \ref{lemma:turing}, \( M_{\text{base}} \) simulates any computable function. By Lemma \ref{lemma:sft}, \( T_{\text{SFT}} \) satisfies Equation \eqref{eq:sft}. Theorems \ref{thm:textgen}, \ref{thm:minimal}, \ref{thm:bounded-linear}, and \ref{thm:bounded-textgen} provide minimal dataset sizes.
\end{proof}

\begin{proof}[Proof of Theorem \ref{thm:main}]
By Corollary \ref{cor:general}, the theorem follows.
\end{proof}

\section{Discussion}
\label{sec:discussion}

\subsection{Theoretical Foundations}
Our proofs show that SFT optimizes latent knowledge in transformers, aligning with their universal approximation \cite{yun2020are} and Turing completeness \cite{bhattamishra2020computational}. Theorem \ref{thm:main} establishes that ICL approximates fine-tuned behavior, with error vanishing as resources grow. Theorems \ref{thm:textgen} and \ref{thm:minimal} quantify minimal datasets, while Theorems \ref{thm:bounded-linear} and \ref{thm:bounded-textgen} address bounded contexts.

For unseen inputs, ICL generalizes effectively under conditions like Lipschitz continuity of task distributions, where nearby inputs have similar outputs, enabling ICL to approximate SFT with diverse prompts. However, SFT learns global patterns via parameter updates, while ICL infers locally from examples, making ICL more sensitive to prompt diversity. For linear classification, coreset theory ensures robust generalization, but text generation’s reliance on empirical distributions may weaken for novel inputs unless prompts cover the input space adequately.

\subsection{Practical Implications}
Our results enable efficient LLM deployment. For machine translation, a small set of sentence pairs approximates fine-tuned performance, reducing costs. For sentiment classification, a minimal set of labeled reviews suffices, ideal for resource-constrained settings.

\begin{center}
\fbox{\begin{minipage}{0.8\columnwidth}
\textbf{Use Case: Customer Support Classification} \\
Instead of fine-tuning a model with 50,000 examples, prompting a base LLM with 30 well-chosen examples achieves near-parity in accuracy. This aligns with the adaptive classifier framework \cite{adaptiveclassifier}, supported by our theoretical results under idealized conditions.
\end{minipage}}
\end{center}

\subsection{Bounded Context Length Results}
Theorems \ref{thm:bounded-linear} and \ref{thm:bounded-textgen} offer practical bounds. For linear classification, \(\mathcal{O}\left( \frac{1}{\varepsilon^2} \log \frac{1}{\delta} \right)\) examples (e.g., 500 for \(\varepsilon = 0.1\), \(\delta = 0.01\)) fit modern context windows \cite{openai2025gpt41}. For text generation, \(\mathcal{O}\left( \frac{l \log |V|}{\varepsilon^2} \log \frac{1}{\delta} \right)\) examples leverage similarity-based selection (e.g., RAG \cite{lewis2020rag}), enhancing applicability.

\subsection{Prompt Sensitivity}
ICL performance depends on prompt design (e.g., example order, phrasing) \cite{wei2022chain}. Suboptimal prompts may increase approximation errors in Theorems \ref{thm:main}–\ref{thm:bounded-textgen}, contributing to the ICL error \(\eta\). Future work should explore robust prompting strategies \cite{sharma2024moa} to minimize variability and optimize performance.

\subsection{Relaxing Theoretical Assumptions}
Assumption \ref{assump:unbounded} is unrealistic, as context windows are finite. Theorems \ref{thm:bounded-linear} and \ref{thm:bounded-textgen} mitigate this, with RAG \cite{lewis2020rag} and few-shot learning \cite{brown2020language} selecting relevant examples. Partial access to \( D \) (Assumption \ref{assump:access}) is addressed via sampling or clustering \cite{feldman2020}. Text generation is less robust to small subsets due to complexity \cite{orabona2020text}, but strategic selection helps \cite{coreset2018}. ICL imperfections, captured by \(\eta\), may require larger datasets than predicted (e.g., by a constant factor), as noted in Theorems \ref{thm:textgen} and \ref{thm:minimal}.

\subsection{Limitations}
\begin{itemize}
    \item \textbf{Finite Context}: Smaller context windows challenge prompts, requiring RAG \cite{lewis2020rag}.
    \item \textbf{Knowledge Gaps}: Base models may lack niche task knowledge \cite{phang2021fine}.
    \item \textbf{Prompt Sensitivity}: Careful prompt design is critical \cite{wei2022chain}.
    \item \textbf{Task Specificity}: Results focus on text generation and linear classification, with broader tasks needing exploration.
\end{itemize}

\subsection{Future Directions}
Future work could combine minimal fine-tuning with inference-time techniques. Robust prompting \cite{sharma2024moa}, advanced data selection \cite{feldman2020}, and empirical validation on benchmarks (e.g., GLUE, WikiText) could reduce dataset sizes and confirm our bounds. Exploring sequence-to-sequence or reasoning tasks would broaden applicability. Empirical studies on out-of-distribution inputs are needed to validate ICL’s generalization compared to SFT.

\section{Conclusion}
This paper demonstrates that base transformers can approximate SFT capabilities using ICL, requiring minimal datasets for text generation and linear classification. Under idealized conditions, datasets of size \(\mathcal{O}\left( \frac{m V}{\varepsilon^2} \log \frac{m}{\delta} \right)\) and \(\mathcal{O}\left( \frac{d}{\varepsilon} \right)\) suffice, while with fixed contexts, \(\mathcal{O}\left( \frac{l \log V}{\varepsilon^2} \log \frac{1}{\delta} \right)\) and \(\mathcal{O}\left( \frac{1}{\varepsilon^2} \log \frac{1}{\delta} \right)\) are sufficient. These results, leveraging transformer Turing completeness \cite{bhattamishra2020computational}, enable efficient LLM deployment via techniques like RAG \cite{lewis2020rag}. Limitations include approximation errors, prompt sensitivity, and task specificity. Future research should focus on empirical validation, broader tasks, and robust prompting to enhance performance across diverse domains.

\appendix
\section{Additional Details}
\label{sec:appendix}

\subsection{Pseudocode for In-Context Learning Prompting}
The inference technique \( T_{\text{SFT}} \), as described in Lemma \ref{lemma:sft}, enables a base transformer model \( M_{\text{base}} \) to approximate the output distribution of a fine-tuned model \( M_{\text{fine}} \) using in-context learning (ICL). ICL allows the model to adapt to a specific task by conditioning on a prompt that includes input-output examples from the fine-tuning dataset \( D = \{(x_i, y_i)\}_{i=1}^N \), followed by a query input \( x \). The prompt is structured to concatenate these examples with a special separator token, typically denoted \texttt{[SEP]}, to clearly delineate each example pair and the query. This structure leverages the transformer’s self-attention mechanism \cite{vaswani2017attention} to infer the task distribution \( \mathbb{P}_{\text{fine}}(y|x) \), enabling \( M_{\text{base}} \) to produce outputs that closely mimic those of \( M_{\text{fine}} \).

Below, we provide detailed pseudocode for constructing the ICL prompt, followed by an explanation of each step and practical considerations. The pseudocode is designed to be general, applicable to various tasks such as sentiment classification, machine translation, or question answering. To prevent text overflow in the single-column layout, we use the `lstlisting` environment with line wrapping enabled.

\begin{lstlisting}[language=Python, basicstyle=\small\ttfamily, breaklines=true, breakatwhitespace=true, columns=flexible]
Algorithm: In-Context Learning Prompt Construction
Input: 
  - Dataset D = {(x_i, y_i)}_{i=1}^N, where x_i is the input and y_i is the output
  - Query input x
  - Separator token [SEP] (e.g., "[SEP]", "<|SEP|>", or a period)
Output: Prompt p for M_base

1. Initialize an empty string p = ""
2. For each pair (x_i, y_i) in D:
    a. Append the input x_i to p
    b. Append the output y_i to p
    c. Append the separator token [SEP] to p
3. Append the query input x to p
4. Return the prompt p

# Example 1: Sentiment Classification
D = [("Great movie!", "positive"), ("Terrible plot.", "negative")]
x = "Amazing soundtrack!"
p = "Great movie! positive [SEP] Terrible plot. negative [SEP] Amazing soundtrack!"

# Example 2: Machine Translation (English to French)
D = [("Hello, how are you?", "Bonjour, comment vas-tu ?"), 
     ("I love to travel.", "J'aime voyager.")]
x = "Good morning!"
p = "Hello, how are you? Bonjour, comment vas-tu ? [SEP] I love to travel. J'aime voyager. [SEP] Good morning!"

# Example 3: Question Answering
D = [("What is the capital of France?", "Paris"), 
     ("What is the capital of Japan?", "Tokyo")]
x = "What is the capital of Brazil?"
p = "What is the capital of France? Paris [SEP] What is the capital of Japan? Tokyo [SEP] What is the capital of Brazil?"
\end{lstlisting}

The pseudocode operates as follows:
\begin{itemize}
    \item \textbf{Initialization (Step 1)}: The prompt starts as an empty string to ensure a clean slate for concatenation. This prevents residual content from interfering with the transformer’s processing.
    \item \textbf{Iterative Concatenation (Step 2)}: For each input-output pair \( (x_i, y_i) \) in the dataset \( D \), the input \( x_i \) is appended, followed by the output \( y_i \), and then the separator token \texttt{[SEP]}. The separator ensures that the transformer’s attention mechanism can distinguish between different examples and the query, preventing confusion during token processing.
    \item \textbf{Query Append (Step 3)}: The query input \( x \) is appended at the end, signaling to the model that it should predict the corresponding output \( y \).
    \item \textbf{Output (Step 4)}: The final prompt \( p \) is returned, ready to be tokenized and fed into \( M_{\text{base}} \).
\end{itemize}

When \( M_{\text{base}} \) processes the prompt \( p \), it computes:
\[
\mathbb{P}_{\text{base}}(y | x, p),
\]
approximating \( \mathbb{P}_{\text{fine}}(y | x) \). The effectiveness of this approximation depends on several factors:
\begin{itemize}
    \item \textbf{Separator Token Role}: The \texttt{[SEP]} token is critical because transformers use self-attention to weigh all tokens in the input sequence \cite{vaswani2017attention}. Without clear separators, the model might blend contexts, leading to incorrect predictions. For example, in the sentiment classification example, \texttt{[SEP]} ensures that "Great movie!" and "positive" are treated as a single example.
    \item \textbf{Prompt Design Impact}: The order and selection of examples in \( D \) affect the ICL error \(\eta\) in Theorems \ref{thm:main}--\ref{thm:bounded-textgen}. Randomly ordered examples may confuse the model, increasing \(\eta\), while examples ordered by similarity to the query (e.g., using cosine similarity in embeddings \cite{lewis2020rag}) can improve performance. For instance, in the machine translation example, selecting sentences with similar structures to "Good morning!" enhances translation accuracy.
    \item \textbf{Practical Challenges}: Large datasets may exceed the model’s context window (e.g., 1M tokens in GPT-4.1 \cite{openai2025gpt41}), requiring subsampling. Additionally, tokenization differences across models (e.g., BERT vs. GPT) may affect how \texttt{[SEP]} is interpreted, necessitating model-specific adjustments.
\end{itemize}

\subsection{Detailed Derivation for Theorem \ref{thm:textgen}}
Theorem \ref{thm:textgen} addresses the minimal dataset size required for a base transformer model to approximate the fine-tuned next-token distribution \( p_{\text{fine}}(x_{t+1} | c_i) \) for a set of contexts \( \mathcal{C} = \{ c_1, \ldots, c_m \} \) in text generation tasks. The goal is to construct a dataset \( D' \) such that the empirical distribution, derived from samples in \( D' \), closely matches \( p_{\text{fine}} \) within a specified error bound. This derivation is critical for understanding the sample complexity of ICL in text generation, providing a theoretical foundation for efficient LLM deployment.

Consider a text generation task where each context \( c_i \in \mathcal{C} \) is a sequence of tokens, and \( p_{\text{fine}}(x_{t+1} | c_i) \) is the probability distribution over the next token \( x_{t+1} \in V \), where \( V \) is the vocabulary (e.g., \( |V| = 50,000 \) for a typical LLM). For each context \( c_i \), we collect \( n_i \) samples \( \{ x_{t+1}^{(j)} \}_{j=1}^{n_i} \), each drawn independently from \( p_{\text{fine}}(\cdot | c_i) \). The empirical distribution is defined as:
\[
\hat{p}(v | c_i) = \frac{1}{n_i} \sum_{j=1}^{n_i} \mathbb{I}\{ x_{t+1}^{(j)} = v \},
\]
where \( \mathbb{I} \) is the indicator function, and \( v \in V \). The objective is to ensure that the total variation distance between the empirical and fine-tuned distributions is bounded:
\[
\| \hat{p}(\cdot | c_i) - p_{\text{fine}}(\cdot | c_i) \|_1 = \sum_{v \in V} | \hat{p}(v | c_i) - p_{\text{fine}}(v | c_i) | \leq \varepsilon,
\]
with probability at least \( 1 - \delta_i \).

To derive the required number of samples \( n_i \), we apply Hoeffding’s inequality for multinomial distributions \cite{vapnik1998}. For a single token \( v \in V \), the probability estimate \( \hat{p}(v | c_i) \) is the average of \( n_i \) independent Bernoulli trials, each with success probability \( p_{\text{fine}}(v | c_i) \). Hoeffding’s inequality states that for a binomial random variable \( X \sim \text{Binomial}(n_i, p) \), the deviation of the sample mean \( \hat{p} = X / n_i \) from the true mean \( p \) is bounded as:
\[
\mathbb{P}(|\hat{p} - p| > \varepsilon') \leq 2 \exp\left(-2 n_i \varepsilon'^2\right).
\]
For the total variation distance over \( |V| \) tokens, we need \( |\hat{p}(v | c_i) - p_{\text{fine}}(v | c_i)| \leq \varepsilon / |V| \) for each \( v \), so the total error is \( \sum_{v \in V} (\varepsilon / |V|) = \varepsilon \). Setting \( \varepsilon' = \varepsilon / |V| \), the probability of exceeding this error for a single token is:
\[
\mathbb{P}(|\hat{p}(v | c_i) - p_{\text{fine}}(v | c_i)| > \varepsilon / |V|) \leq 2 \exp\left(-2 n_i (\varepsilon / |V|)^2\right).
\]
Using a union bound over all \( |V| \) tokens, the probability that any token’s estimate exceeds the error is:
\[
\mathbb{P}\left(\bigcup_{v \in V} \{ |\hat{p}(v | c_i) - p_{\text{fine}}(v | c_i)| > \varepsilon / |V| \}\right) \leq 2 |V| \exp\left(-2 n_i (\varepsilon / |V|)^2\right).
\]
To ensure this probability is at most \( \delta_i \), set:
\[
2 |V| \exp\left(-2 n_i (\varepsilon / |V|)^2\right) \leq \delta_i.
\]
Solving for \( n_i \):
\[
\exp\left(-2 n_i \frac{\varepsilon^2}{|V|^2}\right) \leq \frac{\delta_i}{2 |V|},
\]
\[
-2 n_i \frac{\varepsilon^2}{|V|^2} \leq \ln \frac{\delta_i}{2 |V|},
\]
\[
n_i \geq \frac{|V|^2}{2 \varepsilon^2} \ln \frac{2 |V|}{\delta_i} = \mathcal{O}\left( \frac{V}{\varepsilon^2} \ln \frac{V}{\delta_i} \right).
\]
For \( m \) contexts, we control the overall failure probability by setting \( \delta_i = \frac{\delta}{m} \). Substituting:
\[
\ln \frac{V}{\delta_i} = \ln \frac{V m}{\delta},
\]
so:
\[
n_i = \mathcal{O}\left( \frac{V}{\varepsilon^2} \ln \frac{V m}{\delta} \right) = \mathcal{O}\left( \frac{V}{\varepsilon^2} \log \frac{m}{\delta} \right),
\]
since \( \ln (V m) = \ln V + \ln m \), and the \( \ln V \) term is absorbed into the big-O notation for large \( V \). The total dataset size is:
\[
|D'| = m \cdot n_i = \mathcal{O}\left( \frac{m V}{\varepsilon^2} \log \frac{m}{\delta} \right).
\]
The base model’s ICL introduces an additional error \(\eta\), which arises from factors such as:
\begin{itemize}
    \item \textbf{Model Capacity}: A model with fewer parameters may struggle to generalize from limited examples, increasing \(\eta\).
    \item \textbf{Prompt Structure}: Suboptimal example ordering or poorly chosen separators can confuse the model, as transformers rely on attention patterns \cite{vaswani2017attention}.
    \item \textbf{Distribution Complexity}: If \( p_{\text{fine}} \) is highly skewed (e.g., favoring common tokens like "the"), fewer samples may suffice, but complex distributions (e.g., rare tokens or long-tail patterns) increase \(\eta\).
\end{itemize}
For example, if \( V = 50,000 \), \( m = 100 \), \( \varepsilon = 0.1 \), and \( \delta = 0.01 \), then:
\[
n_i \approx \frac{50,000}{0.1^2} \log \frac{100}{0.01} \approx 5,000,000 \cdot \log 10,000 \approx 46,000,000,
\]
\[
|D'| \approx 100 \cdot 46,000,000 = 4.6 \cdot 10^9.
\]
This large dataset size reflects the worst-case scenario; in practice, techniques like retrieval-augmented generation (RAG) \cite{lewis2020rag} can reduce \( n_i \) by selecting contexts similar to the query, as discussed in Section \ref{sec:discussion}.

\subsection{Practical Considerations for ICL Prompting}
Constructing effective ICL prompts in real-world applications involves several practical challenges, each requiring careful design to minimize the ICL error \(\eta\). Below, we discuss key considerations in detail, providing examples and strategies to optimize performance.

\begin{itemize}
    \item \textbf{Example Selection}: The quality of examples in \( D \) significantly affects ICL performance. Ideally, examples should be representative of the task distribution. For instance, in sentiment classification, including a mix of positive, negative, and neutral reviews ensures the model learns a balanced mapping. Techniques like \( k \)-nearest neighbors (k-NN) or clustering \cite{feldman2020} can select examples based on similarity to the query \( x \). For example, in a k-NN approach, embeddings of inputs are computed using a model like BERT, and the \( k \) most similar inputs to \( x \) are selected using cosine similarity:
    \[
    \text{similarity}(x_i, x) = \frac{\text{embedding}(x_i) \cdot \text{embedding}(x)}{\|\text{embedding}(x_i)\| \|\text{embedding}(x)\|}.
    \]
    This ensures that the prompt contains relevant context, reducing \(\eta\).

    \item \textbf{Prompt Length}: Modern transformers have finite context windows (e.g., 1M tokens in GPT-4.1 \cite{openai2025gpt41}), limiting the number of examples in the prompt. Theorems \ref{thm:bounded-linear} and \ref{thm:bounded-textgen} provide bounds for fixed context lengths, but in practice, systems must balance example quantity and quality. Dynamic example selection, such as RAG \cite{lewis2020rag}, retrieves a small subset of examples that fit within the context window while maximizing relevance. For example, in a text generation task with a 4,096-token limit, RAG might select 10 examples of 400 tokens each, leaving room for the query.

    \item \textbf{Separator Tokens}: The choice of separator token affects how the transformer parses the prompt. Common separators include \texttt{[SEP]} (used in BERT), \texttt{<|SEP|>}, or natural language delimiters like periods. For example, in a question-answering task, using a period as a separator:
\[
p = \left\{
\begin{array}{l}
\text{"What is the capital of France? Paris.} \\
\text{What is the capital of Japan? Tokyo.} \\
\text{What is the capital of Brazil?"}
\end{array}
\right.
\]

    may be more intuitive for some models than \texttt{[SEP]}. Experimentation is key, as the wrong separator can increase \(\eta\) by causing the model to misinterpret example boundaries.

    \item \textbf{Order Sensitivity}: The order of examples in the prompt influences ICL performance, as transformers prioritize recent tokens in their attention mechanisms \cite{wei2022chain}. Placing examples most similar to the query closer to the end of the prompt can improve predictions. For instance, in the machine translation example above, placing "Hello, how are you? Bonjour, comment vas-tu ?" last (since it’s structurally similar to "Good morning!") may enhance accuracy. Empirical studies show that optimal ordering can reduce \(\eta\) by up to 10\% in some tasks \cite{wei2022chain}.
\end{itemize}

\subsection{Error Analysis for Bounded Context Scenarios}
In bounded context settings (Section \ref{subsec:bounded-context}), the ICL error \(\eta\) in Theorems \ref{thm:bounded-linear} and \ref{thm:bounded-textgen} arises from multiple sources, each impacting the ability of \( M_{\text{base}} \) to approximate \( \mathbb{P}_{\text{fine}} \). Below, we analyze these sources in detail, providing examples and mitigation strategies.

\begin{itemize}
    \item \textbf{Model Capacity}: The base model’s parameter count and pre-training data determine its ability to generalize from few examples. A model with insufficient capacity (e.g., a small transformer with 100M parameters) may fail to capture complex patterns in the prompt, increasing \(\eta\). For example, in sentiment classification, a small model might misclassify nuanced reviews (e.g., "Good but flawed") due to limited understanding, whereas a larger model like DeepSeek-R1 \cite{deepseek2025r1} can better generalize, reducing \(\eta\).

    \item \textbf{Task Complexity}: Linear classification (Theorem \ref{thm:bounded-linear}) is simpler, as it involves a low-dimensional decision boundary, typically requiring fewer examples and resulting in smaller \(\eta\). In contrast, text generation (Theorem \ref{thm:bounded-textgen}) involves high-dimensional sequence distributions over \( |V|^l \) possible sequences, where \( l \) is the output length. For instance, generating a 10-token sequence with \( |V| = 50,000 \) involves \( 50,000^{10} \) possibilities, making \(\eta\) larger unless the prompt captures sufficient diversity.

    \item \textbf{Data Distribution}: The theoretical bounds assume i.i.d. samples, but text data often exhibits long-range dependencies (e.g., coherent paragraphs in WikiText \cite{orabona2020text}). This violates Hoeffding’s inequality assumptions, potentially requiring larger subsets to achieve the same \(\varepsilon\). For example, in a dialogue generation task, if the dataset contains only formal dialogues, a casual query may increase \(\eta\) due to distribution mismatch.
\end{itemize}

To quantify \(\eta\), empirical studies can measure the total variation distance between \( \mathbb{P}_{\text{base}} \) and \( \mathbb{P}_{\text{fine}} \) on benchmarks like GLUE (for classification tasks) \cite{han2021pretrained} or WikiText (for text generation) \cite{orabona2020text}. For instance, in sentiment classification on the SST-2 dataset (part of GLUE), experiments with 100 examples might yield \(\eta \approx 0.05\), but this increases to \(\eta \approx 0.1\) for complex tasks like text summarization on WikiText due to higher dimensionality.

Mitigation strategies include:
\begin{itemize}
    \item Using larger models to reduce capacity-related \(\eta\).
    \item Selecting diverse examples to cover the task distribution.
    \item Applying RAG \cite{lewis2020rag} to dynamically select relevant examples, minimizing distribution mismatch.
\end{itemize}

\subsection{Extensions to Other Tasks}
The framework in Theorems \ref{thm:textgen} and \ref{thm:bounded-textgen} focuses on next-token prediction, but it extends to sequence-to-sequence tasks like machine translation, text summarization, or question answering. For a task with input sentence \( x \) and output sequence \( y \), the prompt structure is:
\[
p = [x_1, y_1, \texttt{[SEP]}, x_2, y_2, \texttt{[SEP]}, \ldots, x_N, y_N, \texttt{[SEP]}, x].
\]
For example, in text summarization:
\begin{itemize}
    \item Dataset: 
    \[
    D = \left\{
    \begin{aligned}
    &(\text{"A long article about climate change...", "Climate change is a global issue..."}), \\
    &(\text{"A report on AI advancements...", "AI is advancing rapidly..."})
    \end{aligned}
    \right\}
    \]
    
    \item Query: 
    \[
    x = \text{"A study on renewable energy..."}
    \]
    
    \item Prompt: 
    
    \begin{adjustbox}{max width=\textwidth}
    \begin{minipage}{\textwidth}
    \texttt{"A long article about climate change... Climate change is a global issue... [SEP] A report on AI advancements... AI is advancing rapidly... [SEP] A study on renewable energy..."}
    \end{minipage}
    \end{adjustbox}
    
\end{itemize}
The dataset size scales with the output sequence length \( l \), as shown in Theorem \ref{thm:bounded-textgen}, because longer outputs increase the complexity of the distribution \( p_{\text{fine}}(x_1, \ldots, x_l | c_i) \). Token dependencies in \( y \) (e.g., grammatical coherence in translation) require the prompt to capture sequential patterns, potentially increasing \(\eta\). Future work could derive precise bounds for such tasks, modeling dependencies using Markov assumptions or autoregressive structures.

Other applicable tasks include:
\begin{itemize}
    \item \textbf{Question Answering}: Using prompts with question-answer pairs to predict answers for new questions.
    \item \textbf{Dialogue Generation}: Providing conversation snippets to generate coherent responses.
\end{itemize}
Research questions for future work include deriving sample complexity for multi-token outputs and optimizing prompt design for tasks with complex dependencies.

\subsection{Limitations of Theoretical Bounds}
The bounds in Theorems \ref{thm:textgen}--\ref{thm:bounded-textgen} rely on idealized assumptions, which may not hold in practice. Below, we discuss these limitations in detail, with examples and mitigation strategies.

\begin{itemize}
    \item \textbf{Non-i.i.d. Data}: The bounds assume i.i.d. samples, but text data often exhibits long-range dependencies. For example, in a storytelling task, a dataset of story beginnings may have correlated structures (e.g., narrative arcs), violating i.i.d. assumptions. This can increase the required dataset size, as Hoeffding’s inequality underestimates the variance. Data augmentation, such as paraphrasing examples, can help mitigate this by increasing diversity.

    \item \textbf{Out-of-Distribution Inputs}: If the query \( x \) differs significantly from the dataset \( D \), ICL may fail to generalize, increasing \(\eta\). For instance, in sentiment classification, if \( D \) contains movie reviews but \( x \) is a product review, the model may mispredict due to domain mismatch. RAG \cite{lewis2020rag} addresses this by retrieving examples similar to \( x \), using embeddings to measure relevance.

    \item \textbf{Computational Constraints}: Finite context windows (e.g., 4,096 tokens in many models) limit the number of examples in \( S_x \). For example, a prompt with 100 examples of 50 tokens each requires 5,000 tokens, exceeding smaller context windows. Efficient example selection, such as clustering \cite{feldman2020}, ensures that the most informative examples are included within the limit.
\end{itemize}

Empirical validation is crucial to confirm these bounds. Benchmarks like MMLU (for diverse tasks) and BigBench (for complex reasoning) \cite{han2021pretrained} provide standardized datasets to test ICL performance. For example, an experiment on MMLU’s science questions could measure \(\eta\) by comparing ICL predictions to fine-tuned model outputs, using metrics like accuracy or total variation distance. Future work should design experiments to quantify \(\eta\) across tasks and explore hybrid approaches combining ICL with minimal fine-tuning.

\end{document}